\documentclass[a4paper]{article}

\usepackage{a4wide}
\usepackage{fancyhdr}
\usepackage{comment}
\usepackage[dvipdfmx]{graphicx}
\usepackage{bm}
\usepackage{url}
\usepackage{amsmath}
\usepackage{amsthm}
\usepackage{amssymb}
\usepackage{enumerate}

\newtheorem{thm}{Theorem}[section]
\newtheorem{dfn}{Definition}[section]

\newcommand{\uSet}{\bm{U}}
\newcommand{\bSet}{\bm{B}}
\newcommand{\lSet}{\bm{L}}

\newcommand{\sSet}{\uSet\cup\bSet}
\newcommand{\pSet}{\uSet}
\newcommand{\oSet}{\uSet\cup\bSet\cup\lSet}
\newcommand{\supG}{(\sSet)\times\pSet\times (\oSet)}
\newcommand{\triple}[3]{({#1},{#2},{#3})}

\newcommand{\rSet}[1]{\mbox{\boldmath $R$}_{#1}}

\newcommand{\infofl}[1]{\mathit{Info}({#1})}
\newcommand{\infofr}[1]{-\sum_{c\in C}\frac{|{#1^c}|}{|{#1}|}\log_2\frac{|{#1^c}|}{|{#1}|}}

\newcommand{\ainfofl}[2]{\mathit{Info}\!_{#1}({#2})}
\newcommand{\ainfofr}[2]{\sum_{x\in \{{#1},-{#1}\}}\frac{|{#2^x}|}{|{#2}|}\infofl{#2^x}}

\newcommand{\sinfofl}[2]{\mathit{SplitInfo}\!_{#1}({#2})}
\newcommand{\sinfofr}[2]{-\sum_{x\in \{{#1},-{#1}\}}\frac{|{#2^x}|}{|{#2}|}\log_2\frac{|{#2^x}|}{|{#2}|}}

\newcommand{\igfl}[2]{\mathit{IG}_{#1}({#2})}
\newcommand{\igfr}[2]{\infofl{#1}-\ainfofl{#2}{#1}}

\newcommand{\igrfl}[2]{\mathit{IGR}_{#1}({#2})}
\newcommand{\igrfr}[2]{\frac{\igfl{#1}{#2}}{\sinfofl{#2}{#1}}}

\newcommand{\sppa}{\langle*,\textbf{p},*\rangle}
\newcommand{\spoa}{\langle*,*,\textbf{o}\rangle}
\newcommand{\sppoa}{\langle*,\textbf{p},\textbf{o}\rangle}
\newcommand{\spvpa}{\langle*,*,*,\textbf{p},*\rangle}
\newcommand{\spvoa}{\langle*,*,*,*,\textbf{o}\rangle}
\newcommand{\spvpoa}{\langle*,*,*,\textbf{p},\textbf{o}\rangle}
\newcommand{\spppa}{\langle*,\textbf{p},*,\textbf{p},*\rangle}
\newcommand{\spppoa}{\langle*,\textbf{p},*,\textbf{p},\textbf{o}\rangle}
\newcommand{\sppvoa}{\langle*,\textbf{p},*,*,\textbf{o}\rangle}

\newcommand{\sppb}{\langle*,p,*\rangle}
\newcommand{\spob}{\langle*,*,o\rangle}
\newcommand{\sppob}{\langle*,p,o\rangle}
\newcommand{\spvpb}{\langle*,*,*,p,*\rangle}
\newcommand{\spvob}{\langle*,*,*,*,o^{\prime}\rangle}
\newcommand{\spvpob}{\langle*,*,*,p,o^{\prime}\rangle}
\newcommand{\spppb}{\langle*,p,*,p^{\prime},*\rangle}
\newcommand{\spppob}{\langle*,p,*,p^{\prime},o^{\prime}\rangle}
\newcommand{\sppvob}{\langle*,p,*,*,o^{\prime}\rangle}

\newcommand{\resTableA}{
	\begin{table}[t]
		\begin{center}
			\caption{Classification accuracy for Skip vectors and RDF graph kernels}
			\label{tab:resultA}
		\begin{tabular}{|c||cccccccc|}
	\hline
	& Gender1 & Sea & Net  & Scientist & Box & Gender2 & GDP & Population \\
	&  & Lake & Income & Artist & Office &  &  & Density \\ \hline \hline
	SkipVec-SVM & 0.9250 & 0.9200 & 0.8200 & 0.9900 & 0.9800 & \bf{0.7350} & 0.9200 & 0.8950 \\ 
	SkipVec-KNN & 0.9300 & \bf{0.9800} & 0.8400 & \bf{0.9950} & 0.8450 & 0.6750 & 0.8700 & 0.9050 \\ 
	SkipVec-NN & 0.8950 & 0.9300 & \bf{0.8600} & \bf{0.9950} & 0.9850 & 0.6500 & 0.9500 & 0.9300 \\ 
	SkipVec-RF & \bf{0.9500} & 0.9400 & 0.8400 & 0.9850 & 0.9350 & 0.7100 & 0.9200 & 0.9400 \\ 
	SkipVec-ADA & 0.9450 & 0.9400 & 0.8100 & 0.9700 & \bf{0.9900} & 0.6950 & 0.9300 & \bf{0.9450} \\ 
	\hline
	Skip & 0.8800 & \bf{0.9800} & 0.8200 & 0.9900 & \bf{0.9900} & \bf{0.7350} & 0.9300 & 0.9350 \\
	Hop & 0.7550 & 0.9700 & 0.8100 & \bf{0.9950} & 0.9750 & 0.7100 & 0.9200 & 0.9000 \\
	PRO & 0.9450 & 0.8400 & 0.8300 & 0.9700 & 0.9650 & 0.6950 & \bf{0.9800} & 0.9300 \\
	Walk & 0.6250 & 0.7300 & 0.7900 & 0.7700 & 0.8750 & 0.6050 & 0.8500 & 0.8950 \\
	Path & 0.6250 & 0.7300 & 0.8000 & 0.7850 & 0.8650 & 0.6050 & 0.8600 & 0.8950 \\
	Full SubTree & 0.6300 & 0.7700 & 0.8100 & 0.8650 & 0.9200 & 0.6200 & 0.8300 & 0.8950 \\
	Partial SubTree & 0.5950 & 0.7300 & 0.7900 & 0.8950 & 0.9250 & 0.6050 & 0.8200 & 0.8900 \\
	\hline
\end{tabular}
		\end{center}
	\end{table}
}

\newcommand{\resTableC}{
	\begin{table}[t]
		\begin{center}
			\caption{Number of features extracted by each Skip pattern}
			\label{tab:resultC}
			\begin{tabular}{|c||cccccccc|}
	\hline
	& Gender1 & Sea & Net & Scientist & Box & Gender2 & GDP & Population \\
	&  & Lake & Income & Artist & Office &  &  & Density \\ \hline \hline 
				$F_\textbf{p}$ & 297 & 89 & 144 & 199 & 83 & 22 & 14 & 13 \\
				$F_\textbf{o}$ & 3049 & 1289 & 2223 & 3059 & 5224 & 1575 & 4099 & 1097 \\
				$F_{\textbf{po}}$ & 3122 & 1348 & 2336 & 3334 & 5510 & 1653 & 4214 & 1112 \\
				$F_{*\textbf{p}}$ & 738 & 374 & 601 & 1323 & 519 & 28 & 28 & 20 \\
				$F_{*\textbf{o}}$ & 30618 & 15894 & 14113 & 8752 & 12877 & 5250 & 16508 & 5694 \\
				$F_{*\textbf{po}}$ & 32924 & 16913 & 15409 & 10370 & 15017 & 5295 & 17420 & 5922 \\
				$F_{\textbf{pp}}$ & 3228 & 1350 & 2464 & 3783 & 1498 & 128 & 88 & 68 \\
				$F_{\textbf{ppo}}$ & 37690 & 23866 & 20534 & 13815 & 19734 & 5883 & 23858 & 8540 \\
				$F_{\textbf{p}*\textbf{o}}$ & 36098 & 22937 & 19555 & 12670 & 17554 & 5861 & 23169 & 8372 \\\hline
			\end{tabular}
		\end{center}
	\end{table}
}

\newcommand{\resTableK}{
	\begin{table}[t]
		\begin{center}
			\caption{Classification accuracy for Skip vectors, Feat, WL, RDF2Vec, and R-GCN}
			\label{tab:resultK}
			    \begin{tabular}{|c||cccc|c|}
				\hline
				 & AIFB & MUTAG & BGS & AM & Ave. \\ \hline\hline
				SkipVec-SVM & 0.8611 & 0.7647 & \textbf{0.8966} & 0.9040 & 0.8566 \\ 
				SkipVec-KNN & 0.6389 & 0.6765 & 0.8276 & 0.8687 & 0.7529 \\ 
				SkipVec-NN & 0.8889 & 0.7500 & 0.6552 & 0.9040 & 0.7995 \\ 
				SkipVec-RF & 0.9167 & 0.7206 & \textbf{0.8966} & \textbf{0.9192} & \textbf{0.8633} \\ 
				SkipVec-ADA & 0.8611 & 0.7059 & \textbf{0.8966} & 0.8838 & 0.8369 \\ 
				SkipVec-GCN & \textbf{0.9722} & 0.7353 & - & - & - \\ \hline
				Feat & 0.5555 & 0.7794 & 0.7241 & 0.6666 & 0.6814 \\ 
				WL & 0.8055 & \textbf{0.8088} & 0.8620 & 0.8737 & 0.8375 \\ 
				RDF2Vec & 0.8888 & 0.6720 & 0.8727 & 0.8833 & 0.8292 \\ 
				R-GCN & 0.9583 & 0.7323 & 0.8310 & 0.8929 & 0.8536 \\ \hline
			\end{tabular}
		\end{center}
	\end{table}
}

\title{Skip Vectors for RDF Data: Extraction Based on the Complexity of Feature Patterns} 
\author{Yota Minami\footnotemark[1] \;\;\;\;\;Ken Kaneiwa\footnotemark[1]}
\date{\footnotesize	\footnotemark[1] Department of Computer and Network Engineering, Graduate School of Informatics and Engineering, \\The University of Electro-Communications, Tokyo, Japan}

\begin{document}
	\maketitle
	
	\begin{abstract}
		The Resource Description Framework (RDF) is a framework for describing metadata, such as attributes and relationships of resources on the Web.
		Machine learning tasks for RDF graphs adopt three methods: (i) support vector machines (SVMs) with RDF graph kernels, (ii) RDF graph embeddings, and (iii) relational graph convolutional networks.
		In this paper, we propose a novel feature vector (called a Skip vector) that represents some features of each resource in an RDF graph by extracting various combinations of neighboring edges and nodes.
		In order to make the Skip vector low-dimensional, we select important features for classification tasks based on the information gain ratio of each feature.
		The classification tasks can be performed by applying the low-dimensional Skip vector of each resource to conventional machine learning algorithms, such as SVMs, the $k$-nearest neighbors method, neural networks, random forests, and AdaBoost.
		In our evaluation experiments with RDF data, such as Wikidata, DBpedia, and YAGO, we compare our method with RDF graph kernels in an SVM.
		We also compare our method with the two approaches: RDF graph embeddings such as RDF2Vec and relational graph convolutional networks on the AIFB, MUTAG, BGS, and AM benchmarks.
	\end{abstract} 
	
	\section{Introduction}
	\label{introduction}
	The Resource Description Framework (RDF) is a framework for describing metadata, such as attributes and relationships of resources on the Web.
	The Semantic Web aims to increase the machine-readability of data on the Web by the semantic description of resources in the RDF.
	In fact, RDF datasets from a wide variety of fields are available on the Web as Linked Open Data (LOD)~\cite{10.1007/978-3-319-11964-9_16}.
	In order to make effective use of such RDF data, retrieval, inference, and learning for RDF graphs have been widely investigated in the field of the Semantic Web~\cite{bicer:11, exner:12, fanizzi:12}.
	In particular, machine learning~\cite{arai:18, ristoski16} and data mining for RDF data have been actively studied in recent years.
	
	RDF data has two significant properties that correspond to the description of semantic structures in the Semantic Web.
	The first is that it has a graph structure to facilitate schema-less data integration.
	Due to this property, conventional machine learning algorithms, such as neural networks, cannot directly apply RDF data because the input data are limited to vectors.
	Second, RDF data contains different types of relations and values, such as metadata, ontological data, and attribute data, which are represented together in a graph structure.
	As a semi-supervised learning approach, the graph convolutional network (GCN)~\cite{kipf2016semi} deals with graph data that separates the graph structure from the feature vector of each node.
	However, RDF data is described in a graph structure without separating the feature vector of each node.
	Machine learning tasks for such RDF data adopt three methods: (i) support vector machines (SVMs) with RDF graph kernels, (ii) RDF graph embeddings, and (iii) relational graph convolutional networks (R-GCNs).
	
	In this paper, we present a method for extracting features from various combinations of neighboring edges and nodes in an RDF graph.
	Based on the complexities of features and their patterns, we propose feature vectors (called Skip vectors) of target resources that represent attributes and relationships in the semantic description of RDF data.
	To prevent the dimensionality of Skip vectors from becoming too large, we select important features for classifying resources based on the information gain ratio of each feature.
	For classification tasks in RDF data, the Skip vectors of target resources can be applied to conventional machine learning algorithms, such as SVMs, the $k$-nearest neighbors (KNN) method, neural networks, random forests (RFs), and AdaBoost (ADA).
	In the evaluation experiments with RDF data such as Wikidata, DBpedia, and YAGO, we compare our method with RDF graph kernels in SVMs.
	On the AIFB, MUTAG, BGS, and AM benchmarks, we compare our method with two approaches: RDF graph embeddings such as RDF2Vec and relational graph convolutional networks.
	
	The remainder of this paper is organized as follows.
	In Section 2, we introduce the basic concepts of RDF graphs, the feature extraction of RDF resources, and the information gain ratio.
	In Section 3, we theoretically analyze the complexities of features and feature patterns for each resource in RDF graphs, and present a method for generating the Skip vectors as feature vectors based on feature patterns.
	In Section 4, we provide the evaluation experiments of our Skip vectors with conventional machine learning algorithms.
	In Section 5, we discuss related works.
	Finally, in Section 6, we conclude this paper and discuss future work.
	
	
	\section{Preliminary}  
	\label{preparation}
	
	\subsection{RDF Graph}
	Let $\uSet$ be a set of Uniform Resource Identifier (URI) references, $\lSet$ be a set of literals, and $\bSet$ be a set of empty nodes.
	An RDF triple is a tuple $(s,p,o)$ of subject $s$, predicate $p$, and object $o$ with $s \in \sSet$, $p \in \pSet$, and $o \in \oSet$. 
	In other words, an RDF triple is an element of $\supG$.
	An RDF graph $G$ is defined as a finite set of RDF triples \{$\triple{s_1}{p_1}{o_1} ,\cdots ,\triple{s_n}{p_n}{o_n}$\}. 
	
	For example, an RDF graph for describing fruits and animals is shown in Figure \ref{fig:rdf}.
	\begin{figure}[t]
		\centering
		\includegraphics[width=10cm]{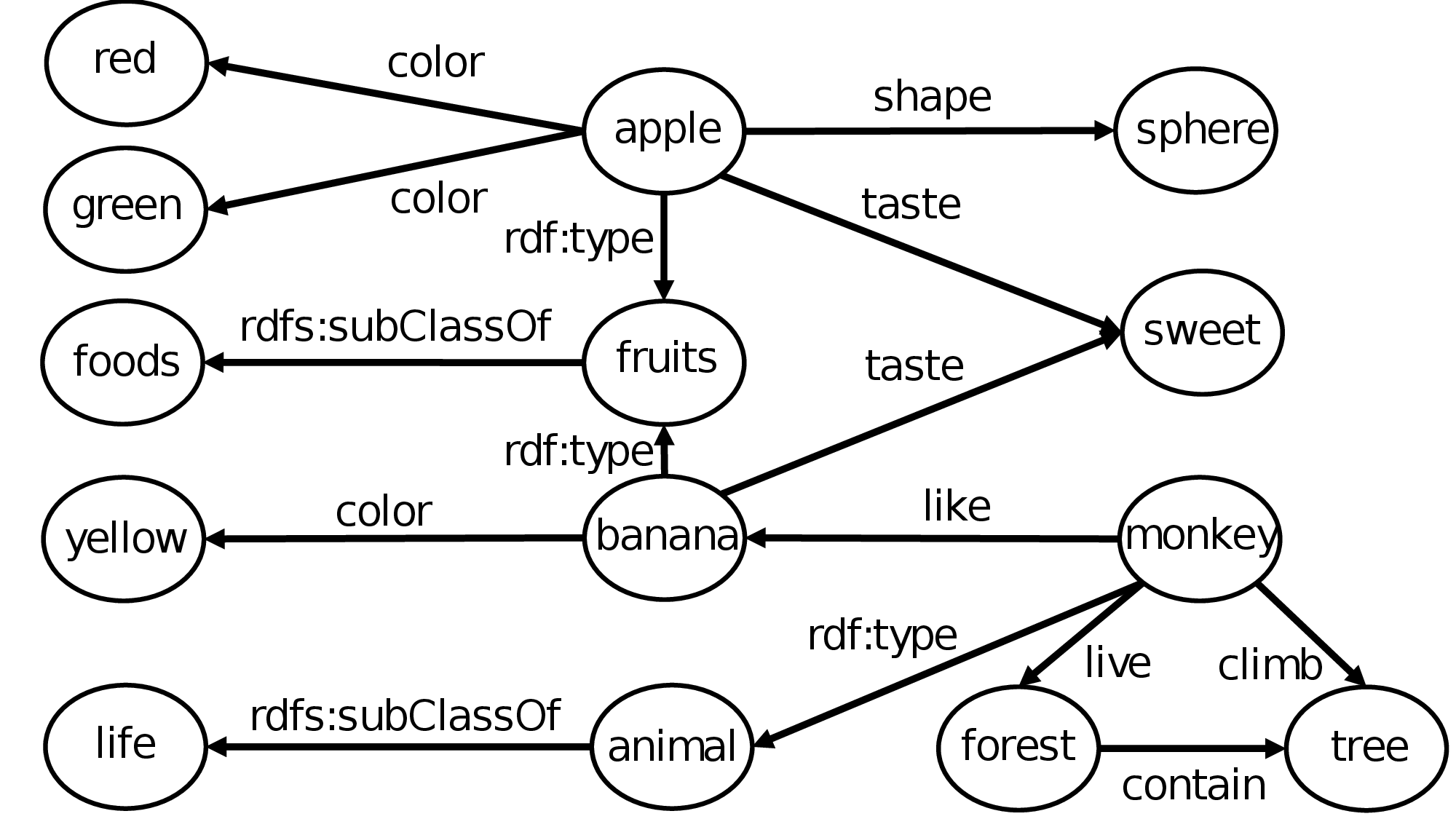}
		\caption{Example of an RDF graph}
		\label{fig:rdf}
	\end{figure}
	
	\subsection{Feature Extraction for RDF Resources}
	As a feature extraction for graph and tree structures, kernel functions calculate the distance between data by counting the common substructures in two graphs~\cite{fanizzi2006declarative, huang2014scalable, collins:01, shervashidze2011weisfeiler, vishwanathan2010graph}. 
	Specifically, for RDF graphs, the Walk, Path, and Subtree kernels~\cite{losch12}, which are based on intersection graphs and trees, have been devised by L{\"o}sch et al. Moreover, the PRO kernel~\cite{arai:17} and Skip kernel~\cite{arai:18}, which are extensions of these kernel functions, have been proposed by Arai et al.
	
	A simple path that exists in the neighborhood of a resource in an RDF graph $G$ is called a Walk.

	\begin{dfn}[Walk]
		A sequence of $d$ triples whose object and subject match
		\begin{eqnarray*}
			\triple{s}{p_1}{o_1},\triple{o_1}{p_2}{o_2},...,\triple{o_{d-1}}{p_d}{o_d}
		\end{eqnarray*}
		is called a Walk of depth $d$.
		It is simply a sequence of $2d+1$ elements starting from the subject $s$, which can also be written as
		\begin{eqnarray*}
			\langle s,p_1,o_1,p_2,o_2,...,p_d,o_d \rangle
		\end{eqnarray*}
		In particular, a Walk of depth $0$ starting at resource $s$ is $\langle s \rangle$.
	\end{dfn}
	
	Let $X$ be the space of data points, and $F$ be the space of their features.
	The mapping $ex:X \rightarrow F$ from a data point to a feature is called a feature extraction.
	Features in an RDF graph $G$ are extracted by using a Walk starting from the target resource.
	Let $Walk(s)$ be the set of Walks starting from resource $s$ on the RDF graph $G$, and let $EX$ be the set of feature extractions.
	The feature set $F_{walk}(s)$ of Walks for resource $s$ is expressed as follows:
	\begin{eqnarray*}
		F_{walk}(s) = \{ ex(w) \: | \: w \in Walk(s), ex \in EX \}
	\end{eqnarray*}
	
	The partial map $v:\mathbb{N} \rightarrow \{*\}$ that transforms any position $i$ of Walk $\langle e_1, e_2, ..., e_n \rangle$ into a variable $*$, as in $v(i) = *$, is called a {\bf skip function}.
	We denote the set of skip functions by $Var$.
	A skip function $v \in Var$ is extended to a function on elements $e_i$ and element sequences $\langle e_1, ..., e_n \rangle$ as follows:
	\begin{eqnarray*}
		&&v(e_i) = \begin{cases}
			e_i & \quad (i \notin Dom(v)) \\
			v(i) & \quad (i \in Dom(v))
		\end{cases}
		\\
		&&v(\langle e_1,...,e_n \rangle) = \langle v(e_1),...,v(e_n) \rangle
	\end{eqnarray*}

	A PRO is a feature in which there is no object $o_1,o_2,...,o_{d-1}$ on the path of a Walk of depth $d$.
	\begin{dfn}[PRO]
		The feature set $F_{pro}$ of PROs for resource $s$ is defined as follows.
		\begin{eqnarray*}
			F_{pro}(s) = \{ v(w) \: | \: w \in Walk(s), v \in Var, Dom(v)=\{1,3,...,2|w|-1\} \}
		\end{eqnarray*}
	\end{dfn}

	A Skip is a feature in which arbitrary predicates and objects are made variables and missing from a PRO.
	\begin{dfn}[Skip]
		For any $pro \in F_{pro}(s)$, the feature set $F_{skip}(s)$ of Skips for resource $s$ is defined as follows.
		\begin{eqnarray*}
			F_{skip}(s) = \{ v(pro) \: | \: pro \in F_{pro}(s), v \in Var, \{2|pro|,2|pro|+1\} \backslash \{0\} \nsubseteq Dom(v) \}
		\end{eqnarray*}
	\end{dfn}
	\hspace{-5mm}where $|pro|$ is the depth of $pro$.
	
	\subsection{Information Gain Ratio}
	Let $G$ be an RDF graph.
	A set of resources as training data is denoted by $\rSet{G}\subset\uSet$.
	Let $C$ be a set of class labels.
	The set of resources with class label $c\in C$ is denoted as $\rSet{G}^c\subseteq \rSet{G}$.
	Let $F$ be the set of all features in the RDF graph $G$. The set of resources that possess feature $f\in F$ is denoted by $\rSet{G}^f$, and the set of resources that do not possess feature $f$ is denoted by $\rSet{G}^{-f}$.
	
	For an RDF graph $G$, the information entropies $\infofl{\rSet{G}}$ for $\rSet{G}$ and $\ainfofl{f}{\rSet{G}}$ for the resources partitioned by a feature $f$ are defined as follows:
	\begin{eqnarray*}
		\infofl{\rSet{G}}&=&\infofr{\rSet{G}}\\
		\ainfofl{f}{\rSet{G}}&=&\ainfofr{f}{\rSet{G}}\\
	\end{eqnarray*}
	The split information entropy $\sinfofl{f}{\rSet{G}}$, which is the information entropy with feature $f$ as a class label; the information gain $\igfl{\rSet{G}}{f}$, which represents the decrease in information entropy if data is split by feature $f$; and the information gain ratio $\igrfl{\rSet{G}}{f}$, which is the normalized information gain, are defined as follows.
	\begin{eqnarray*}
		\sinfofl{f}{\rSet{G}}&=&\sinfofr{f}{\rSet{G}}\\
		\igfl{\rSet{G}}{f}&=&\igfr{\rSet{G}}{f}\\
		\igrfl{\rSet{G}}{f}&=&\igrfr{\rSet{G}}{f}
	\end{eqnarray*}
	
	\section{Feature Vectors from an RDF Graph}
	In this section, we define the representation vector of each resource that is obtained from its feature set based on a pattern of features in an RDF graph.
	
	\subsection{Complexity of Features}
	In Section 2.2, Walks were defined as the simple paths of each resource on an RDF graph.
	In addition, PROs and Skips are defined as feature extractions that take into account the nature of RDF data.
	As a more simplified variant of Walks, PROs are features consisting of the predicate sequence $p_1,...,p_n$ and the terminal object $o$ without the intermediate nodes (objects), while Skips are features that extract the substructures of PROs by skipping any predicates and objects.
	
	First, we theoretically analyze each feature extracted on the RDF graph by counting the combinations of predicates and objects.
	The following theorem shows the complexities of features starting from a subject resource (node) in Subtrees, Walks, Paths, PROs, and Skips.
	
	\begin{thm}[Complexity of features]
		\label{thm3.1}
		Let $m$ and $n$ be the total number of predicates and objects in an RDF graph $G$, respectively.
		A feature set from depth $1$ to $d$ for a resource satisfies the following property.
		\begin{description}
			\item[(i)] The number of Subtrees is in $O(2^{(mn)^d})$.
			\item[(ii)] The number of Walks and Paths is in $O((mn)^d)$.
			\item[(iii)] The number of PROs is in $O(m^{d}n)$.
			\item[(iv)] The number of Skips is in $O(m^{d}n)$.
		\end{description}
	\end{thm}
	\begin{proof}
		\textbf{(i)} If the depth $d = 1$, the number of Subtrees with depth 1 is $ \sum_{k=1}^{mn} {}_{mn} \mathrm{C}_k = 2^{mn} - 1$, where the number of predicate-object pairs is $ m\times n$, and all predicate-object combinations ${}_{mn} \mathrm{C}_{1} \sim {}_{mn} \mathrm{C}_{mn}$ are features of Subtrees.
		The number of predicate-object combinations of depth 2 that can be reached from a predicate-object pair of depth 1 is $2^{mn}$, and by adding the case of only a predicate-object pair of depth 1, the number of this part is $2^{mn}+1$. 
		Since the combinations of the $mn$ parts are all Subtrees, the number of Subtrees up to depth 2 is $(2^{mn}+1)^{mn}-1$, except when there are not all predicates and objects.
		The number of predicate-object combinations of depth 3 that can be reached from predicate-object pairs of depth 1 and 2 is $2^{mn}$, and considering the case with only predicate-objects of depth 1 and 2, and the case with only predicate-objects of depth 1, the number of this part is $2^{mn}+1+1=2^{mn}+2$. 
		Since the combinations of the $(mn)^2$ parts are all Subtrees, the number of Subtrees up to depth 3 is $(2^{mn}+2)^{mn^2}-1$, except when there are not all predicates and objects.
		Thus, the number of Subtrees up to depth of $d$ is $(2^{mn}+(d-1))^{(mn)^{(d-1)}}-1$. 
		Therefore, the number of subtrees is in $O(2^{(mn)^d})$.	
			
		\textbf{(ii)} Because Walks and Paths are sequences of predicate-object pairs, the number of Walks and Paths is $m\times n$ at depth 1, $(m\times n)^2$ at depth 2, and $(m\times n)^3$ at depth 3, and the number of features at depth $d$ is $(m\times n)^d$.
		Thus, the number of Walks and Paths up to depth of $d$ is $mn+(mn)^2+\cdots+(mn)^n$.
		Therefore, the number of Walks and Paths is in $O(2^{(mn)^d})$.
		
		\textbf{(iii)} Because PROs are features that skip all objects from the Walk except for the endpoint, the number of PROs is $m \times n$ at depth 1, $m^2 \times n$ at depth 2, and $m^3 \times n$ at depth 3, and the number of features at depth $d$ is $m^d \times n$.
		Thus, the number of PROs up to depth of $d$ is $mn+m^2n+\cdots+m^dn$.
		Hence, the number of PROs is in $O(m^dn)$.
		
		\textbf{(iv)} Skips are features that skip any predicate and object from PROs.
		Because Skips of depth 1 consist of predicates (e.g., $\langle s,p_1,* \rangle$), objects (e.g., $\langle s,*,o_1 \rangle$), and PROs of depth 1 (e.g., $\langle s,p_1,o_1 \rangle$), the number of Skips is $m + n + mn$.
		The number of Skips to depth 2 is $m + n + mn + mm + mn + m^2n$.
		This is because Skips to depth 2 consist of predicates of depth 2 (e.g., $\langle s,*,*,p_2,* \rangle$), objects of depth 2 (e.g., $\langle s,*,*,*,o_2 \rangle$), predicate and object of depth 2 pairs (e.g., $\langle s,*,*,p_2,o_2 \rangle$), predicates of depth 1 and depth 2 pairs (e.g., $\langle s,p_1,*,p_2,* \rangle$), predicate of depth 1 and object of depth 2 pairs (e.g., $\langle s,p_1,*,*,o_2 \rangle$), and PROs of depth 2 (e.g., $\langle s,p_1,*,p_2,o_2 \rangle$).  
		Therefore, the number of Skips to depth 2 is $(m + n + mn + mm + mn + m^2n) + (m + n + mn)$.
		Let the number of Skips up to the depth $d$ be $S_d$; then $S_d$ can be expressed as $S_1=m+n+mn$ and $S_d = S_{d-1} + mS_{d-1} + S_{d-1} = (2+m)S_{d-1}$. 
		Thus, $S_{d} = S_1 \cdot (2+m)^{d-1} = (m+n+mn) \cdot (2+m)^{d-1} = (m+n(1+m)) \cdot (2+m)^{d-1} = m \cdot (2+m)^{d-1} + n \cdot (1+m) \cdot (2+m)^{d-1}$.
		Therefore, the number of Skips is in $O(m^{d}n)$.
	\end{proof}
	
	In practice, it is difficult to compute all the Subtrees that produce very large combinations in $O(2^{(mn)^d})$.
	In RDF data, the number of predicates $m$, the number of objects $n$, and the depth $d$ are at very different scales, where $n$ is on the order of 10,000, $m$ is roughly 10 to 100, and $d$ is approximately 2 to 5, respectively.
	Since the number of predicates $m$ is much smaller than the number of objects $n$ (i.e., $m \ll n$), $m^d$ is computationally small.
	Therefore, $O(m^{d}n)$ for PROs and Skips is properly less complex than $O((mn)^d)$ for Walks and Paths.
	In addition to the complexity, Skips are highly expressive as an extension of PROs.
	As a result, the Skips can achieve high accuracy in classifying resources, while suppressing combinatorial explosions such as Subtrees.
	
	\subsection{Complexity of Feature Patterns}
	We carry out the following method to limit the exponential combinations with respect to the depth of features from each resource.
	
	\begin{itemize}
		\item The features of each resource are categorized into patterns that are divided into predicate and object components (called feature patterns).
		\item The search cost of the RDF graph is reduced by limiting the feature depth and the number of feature patterns.
	\end{itemize}
	
	We classify patterns of the feature sets Subtrees, Walks, Paths, PROs, and Skips by focusing on the types of predicates and objects.
	The elements in each feature set are rewritten as predicate type \textbf{p}, object type \textbf{o}, and variable $*$, which are called feature patterns.
	For example, two features $\langle*,p_1,*,p_2,o_1\rangle$ and $\langle*,p_3,*,p_4,o_2\rangle$ belong to the same feature pattern $\langle*$,\textbf{p},$*$,\textbf{p},\textbf{o}$\rangle$.
	With this feature pattern, each feature can be grouped according to its structure and variables, and the types of predicates and objects.
	
	The following theorem shows the complexities of feature patterns (called Subtree, Walk, Path, PRO, and Skip patterns).
	\begin{thm}[Complexity of feature patterns]
		\label{thm3.2}
		Let the total number of predicates and objects in an RDF graph $G$ be $m$ and $n$, respectively.
		A feature pattern from depth $1$ to $d$ for a resource satisfies the following properties.
		\begin{description}
			\item[(i)] The number of Subtree patterns is in $O(2^{n^{d}})$.
			\item[(ii)] The number of Walk and Path patterns is in $O(d)$.
			\item[(iii)] The number of PRO patterns is in $O(d)$.
			\item[(iv)] The number of Skip patterns is in $O(2^{d})$.
		\end{description}
	\end{thm}
	\begin{proof}
		\textbf{(i)} The number of Subtree patterns at depth 1 is $n$, because each Subtree pattern has at most $n$ Walks.
		There are $n$ Subtree patterns of depth 2 reaching from a Walk of depth 1, only a Walk of depth 1, and not also a Walk.
		Then, the number for this part is $n+1+1=n+2$.
		There are $n$ of these parts, but if we do not count the Subtree patterns once they are chosen, the combination of these parts is $(n+2)\cdot(n+1)\cdots 2 \cdot 1$.
		Because this is all the Subtree patterns, the number of Subtrees up to depth 2 is $(n+2)!$.
		There are $n$ Subtree patterns of depth 3 reaching from a Walk of depth 2, only a Walk of depth 2, only a Walk of depth 1, and not also a Walk.
		Then, the number for this part is $n+1+1=1=n+3$.
		There are $n^2$ of these parts, but if we do not count the Subtree patterns once they are chosen, the combinations for these parts is $((n+3)^n)!$.
		Because this is all the Subtree patterns, the number of Subtrees up to depth 3 is $((n+3)^n)! $.
		Thus, the number of Subtree patterns up to depth of $d$ is $((n+d)^{n^{d-2}})! \approx (n+d)^{n^{d-1}}$.
		Because $n^{n^{d-1}} = 2^{\log n\cdot n^{d-1}}$, the number of Subtree patterns is in $O(2^{n^{d}})$.
		
		\textbf{(ii)} The number of Walk and Path patterns increases by one for each additional depth: \textbf{po} at depth 1, \textbf{popo} at depth 2, and \textbf{popopo} at depth 3. 
		Thus, the number of Walk and Path patterns up to depth $d$ is $\sum_{k=1}^d 1 = d$. 
		Hence, the number of Walk and Path patterns is in $O(d)$.
		
		\textbf{(iii)} The number of PRO patterns increases by one for each additional depth: \textbf{po} at depth 1, \textbf{ppo} at depth 2, \textbf{pppo} at depth 3. 
		Thus, the number of PRO patterns up to depth $d$ is $\sum_{k=1}^d 1 = d$. 
		Therefore, the number of PRO patterns is in $O(d)$.
		
		\textbf{(iv)} Skip patterns are constructed by arbitrarily dropping predicates and objects from PRO patterns.
		The number of skip patterns increases by a factor of two as the depth increases, from three $\textbf{p}$,$\textbf{o}$,$\textbf{po}$ at depth 1 to six $*\textbf{p}$,$*\textbf{o}$,$*\textbf{po}$,$\textbf{pp}$,$\textbf{ppo}$,$\textbf{p}\text{\textasteriskcentered} \textbf{o}$ at depth 2. 
		Therefore, the number of Skip patterns at depth $d$ can be expressed as $3 \cdot 2^{d-1}$.
		Thus, the number of Skip patterns up to depth $d$ is $\sum_{k=1}^d 3 \cdot 2^{d-1} = \frac{3(2^d-1)}{2-1} = 3 \cdot 2^{d}-3$. 
		So, the number of skip patterns is in $O(2^{d})$.
	\end{proof}
	
	The Walk and PRO patterns have only $d$ feature patterns, depending on the depth $d$, while the Subtree patterns have a large number of patterns in $O(2^{n^{d}})$, depending on the number of objects $n$, which is difficult to compute.
	The complexity of Skip patterns is an exponential order $O(2^{d})$, as in Theorem \ref{thm3.2} (iv).
	This is because Skips are generalized from PROs, where variables are introduced into the PROs to increase the expressiveness of features.
	Fortunately, it does not depend on the number of predicates $m$ and the number of objects $n$, which have large values.
	This means that if the depth $d$ is limited to a small value (e.g., $d=2$), then the number of Skip patterns is limited to a few (e.g., $3\cdot 2^d -3 = 3\cdot 4 - 3 = 9$), which indicates that Skip patterns are computable in practice.
	
	\subsection{Feature Sets Based on Skip Patterns}
	
	\begin{table}[t]
		\begin{center}
			\caption{Skip patterns with depth $d \le 2$}
			\label{tab:Skippattern}
			\scalebox{0.9}{
				\begin{tabular}[t]{|c|c|c|c|}\hline
					Abbreviation & Skip Pattern & Abbreviation & Skip Pattern \\ \hline
					$\textbf{p}$ & $\sppa$ & $*\textbf{p}$ & $\spvpa$ \\
					$\textbf{o}$ & $\spoa$ & $*\textbf{o}$ & $\spvoa$ \\
					$\textbf{po}$ & $\sppoa$ & $*\textbf{po}$ & $\spvpoa$ \\
					& & $\textbf{pp}$ & $\spppa$ \\
					& & $\textbf{ppo}$ & $\spppoa$ \\
					& & $\textbf{p}$$*$$\textbf{o}$ & $\sppvoa$ \\
					\hline
				\end{tabular}
			}
		\end{center}
	\end{table}
	
	Table \ref{tab:Skippattern} summarizes the Skip patterns of depth $d \leq 2$.
	Each Skip pattern is expressed by a feature pattern, where \textbf{o} and \textbf{p} indicate the object and predicate types, and the variable \textasteriskcentered \, means to skip one predicate or object.
	For classification tasks of resources $s\in \rSet{G}$, we define simple feature sets based on the nine Skip patterns $\mathit{SP}=$ \{$\textbf{p}$, $\textbf{o}$, $\textbf{po}$,  $*\textbf{p}$, $*\textbf{o}$, $*\textbf{po}$, $\textbf{pp}$, $\textbf{ppo}$, $\textbf{p}\text{\textasteriskcentered} \textbf{o}$\}.
	Namely, $\mathit{sp} \in \{ \textbf{p}, \textbf{o}, \textbf{po} \}$ has a depth of $1$ and $\mathit{sp} \in \{ *\textbf{p},*\textbf{o},*\textbf{po},\textbf{pp},\textbf{ppo}, \textbf{p}\text{\textasteriskcentered} \textbf{o} \}$ has a depth of $2$.
	The depth $d$ of each Skip pattern $\mathit{sp}\in \mathit{SP}$ indicates the number of triples that are extracted from each resource $s$ in an RDF graph $G$. 
	
	\begin{dfn}[Feature sets of objects]
		The feature sets of objects for each resource $s\in \rSet{G}$ are defined as follows:
		\begin{eqnarray*}
			F_{\mathbf{o}}(s)&=&\{\spob \in F_{skip}(s) \mid (s,p,o)\in G\}\\
			F_{*\mathbf{o}}(s)&=&\{\spvob \in F_{skip}(s) \mid(o,p,o^{\prime})\in G, \spob \in F_{\mathbf{o}}(s)\}
		\end{eqnarray*}
	\end{dfn}
	For the example in Figure \ref{fig:rdf}, the feature sets of objects for the resource ``apple'' are represented as follows:
	\begin{eqnarray*}
		F_{\mathbf{o}}(\text{apple})&=&\{\langle *,*,\text{red}\rangle, \langle *,*,\text{green}\rangle, \langle *,*,\text{fruits}\rangle, \langle *,*,\text{sphere}\rangle, \langle *,*,\text{sweet}\rangle\}\\
		F_{*\mathbf{o}}(\text{apple})&=&\{\langle *,*,*,*,\text{foods}\rangle\}
	\end{eqnarray*}
	
	\begin{dfn}[Feature sets of predicates and objects]
		The feature sets of predicates and objects for each resource $s\in \rSet{G}$ are defined as follows:
		\begin{eqnarray*}
			F_{\mathbf{po}}(s)&=&\{\sppob \in F_{skip}(s) \mid (s,p,o)\in G\}\\
			F_{*\mathbf{po}}(s)&=&\{\spvpob \in F_{skip}(s) \mid(o,p,o^{\prime})\in G, \spob \in F_{\mathbf{o}}(s)\}\\
			F_{\mathbf{ppo}}(s)&=&\{\spppob \in F_{skip}(s) \mid(o,p^{\prime},o^{\prime})\in G, \sppob \in F_{\mathbf{po}}(s)\}\\
			F_{\mathbf{p}*\mathbf{o}}(s)&=&\{\sppvob \in F_{skip}(s) \mid(o,p^{\prime},o^{\prime})\in G, \sppob \in F_{\mathbf{po}}(s)\}
		\end{eqnarray*}
	\end{dfn}
	For the example in Figure \ref{fig:rdf}, the feature sets of predicates and objects for the resource ``apple'' are represented as follows:
	\begin{eqnarray*}
		F_{\mathbf{po}}(\text{apple})&=&\{\langle *,\text{color, green}\rangle, \langle *,\text{color, red}\rangle, \langle *,\text{shape, sphere}\rangle, \langle *,\text{taste, sweet}\rangle, \\&& \ \ \langle *,\text{rdf:type, fruits}\rangle\}\\
		F_{*\mathbf{po}}(\text{apple})&=&\{\langle *,*,*,\text{rdfs:subClassOf, foods}\rangle \}\\
		F_{\mathbf{ppo}}(\text{apple})&=&\{\langle *,\text{rdf:type},*,\text{rdfs:subClassOf}, \text{foods}\rangle \}\\
		F_{\mathbf{p}*\mathbf{o}}(\text{apple})&=&\{\langle *,\text{rdf:type},*,*,\text{foods}\rangle \}
	\end{eqnarray*}
	
	\begin{dfn}[Feature sets of predicates]
		The feature sets of predicates for each resource $s\in \rSet{G}$ are defined as follows:
		\begin{eqnarray*}
			F_\mathbf{p}(s)&=&\{\sppb \in F_{skip}(s) \mid (s,p,o)\in G\}\\
			F_{*\mathbf{p}}(s)&=&\{\spvpb \in F_{skip}(s) \mid(o,p,o^{\prime})\in G, \spob\in F_{\mathbf{o}}(s)\}\\
			F_{\mathbf{pp}}(s)&=&\{\spppb \in F_{skip}(s) \mid(o,p^{\prime},o^{\prime})\in G, \sppob\in F_{\mathbf{po}}(s)\}
		\end{eqnarray*}
	\end{dfn}
	For the example in Figure \ref{fig:rdf}, the feature sets of predicates for the resource ``apple'' are represented as follows:
	\begin{eqnarray*}
		F_\mathbf{p}(\text{apple})&=&\{\langle *,\text{color},* \rangle, \langle *,\text{shape},* \rangle, \langle *,\text{taste},* \rangle, \langle *,\text{rdf:type},* \rangle \}\\
		F_{*\mathbf{p}}(\text{apple})&=&\{\langle *,*,*,\text{rdfs:subClassOf},* \rangle \}\\
		F_{\mathbf{pp}}(\text{apple})&=&\{\langle *,\text{rdf:type},*,\text{rdfs:subClassOf},*\rangle \}
	\end{eqnarray*}
	
	Let $\mathit{sp} \in \mathit{SP}$ be a Skip pattern.
	The feature set of $\mathit{sp}$ for all resources $s\in \rSet{G}$ is represented as follows:
	$$F_{\mathit{sp}}(\rSet{G})=\bigcup_{s\in \rSet{G}}F_{\mathit{sp}}(s)$$
	
	\subsection{Skip Vectors}
	A Skip vector is a vector representation of the feature set $F_{\mathit{sp}}(s)$ of $\mathit{sp} \in \mathit{SP}$ for resource $s$.
	\begin{dfn}[Skip Vectors]
		Given a feature sequence $[f_1,f_2,\dots ,f_n]$ in $F_{\mathit{sp}}(\rSet{G})=\{f_1,f_2,\\\dots ,f_n\}$, the Skip vector $V_{\mathit{sp}}(s)$ of $\mathit{sp} \in \mathit{SP}$ for resource $s\in \rSet{G}$  is defined as follows:
		$$V_{\mathit{sp}}(s)=[x_{1}^{\mathit{sp}},x_{2}^{\mathit{sp}},\dots ,x_{n}^{\mathit{sp}}]$$
		where each $x_{i}^{\mathit{sp}}$ is defined by:
		\begin{eqnarray*}
			x_{i}^{\mathit{sp}} = \begin{cases}
				1 & \quad \mathrm{if} \; f_i\in F_{\mathit{sp}}(s) \\
				0 & \quad \mathrm{otherwise}
			\end{cases}
		\end{eqnarray*}
	\end{dfn}

	The concatenation of two Skip vectors $V_{\mathit{sp}_1}(s)$ and $V_{\mathit{sp}_2}(s)$ is denoted as $V_{\mathit{sp}_1}(s) \parallel V_{\mathit{sp}_2}(s)$.
	The concatenation of Skip vectors $V_{\mathit{sp}}(s)$ by all $\mathit{sp}\in \mathit{SP}$ is denoted as $V_{all}(s) = \parallel_{\mathit{sp}\in \mathit{SP}} V_{\mathit{sp }}(s)$.
	We abbreviate $V_{\mathit{sp}}(s)$ and $V_{all}(s)$ as $V_{\mathit{sp}}$ and $V_{all}$ if the resource $s$ is not explicitly specified.
	
	In the example of Figure \ref{fig:rdf}, we have the set of resources $\rSet{G}$ $=$ $\{\text{apple, } \text{monkey}\}$ and the feature sequence [$\langle * $,color,$ * \rangle$, $\langle *$,shape,$* \rangle$, $\langle *$,rdf:type,$* \rangle$, $\langle *$,like,$* \rangle$, $\langle *$,live,$* \rangle$, $\langle *$,taste,$* \rangle$, $\langle *$,climb,$* \rangle$] in $F_\textbf{p}(\rSet{G})$.
	Then, the Skip vector of $F_\textbf{p}(\text{apple})$ for the resource ``apple'' is generated as:
	
	$$V_\textbf{p}(\text{apple})=[1,1,1,0,0,1,0]$$
	
	\begin{figure}
		\centering
		\hspace{0cm}
		\includegraphics[width=12cm]{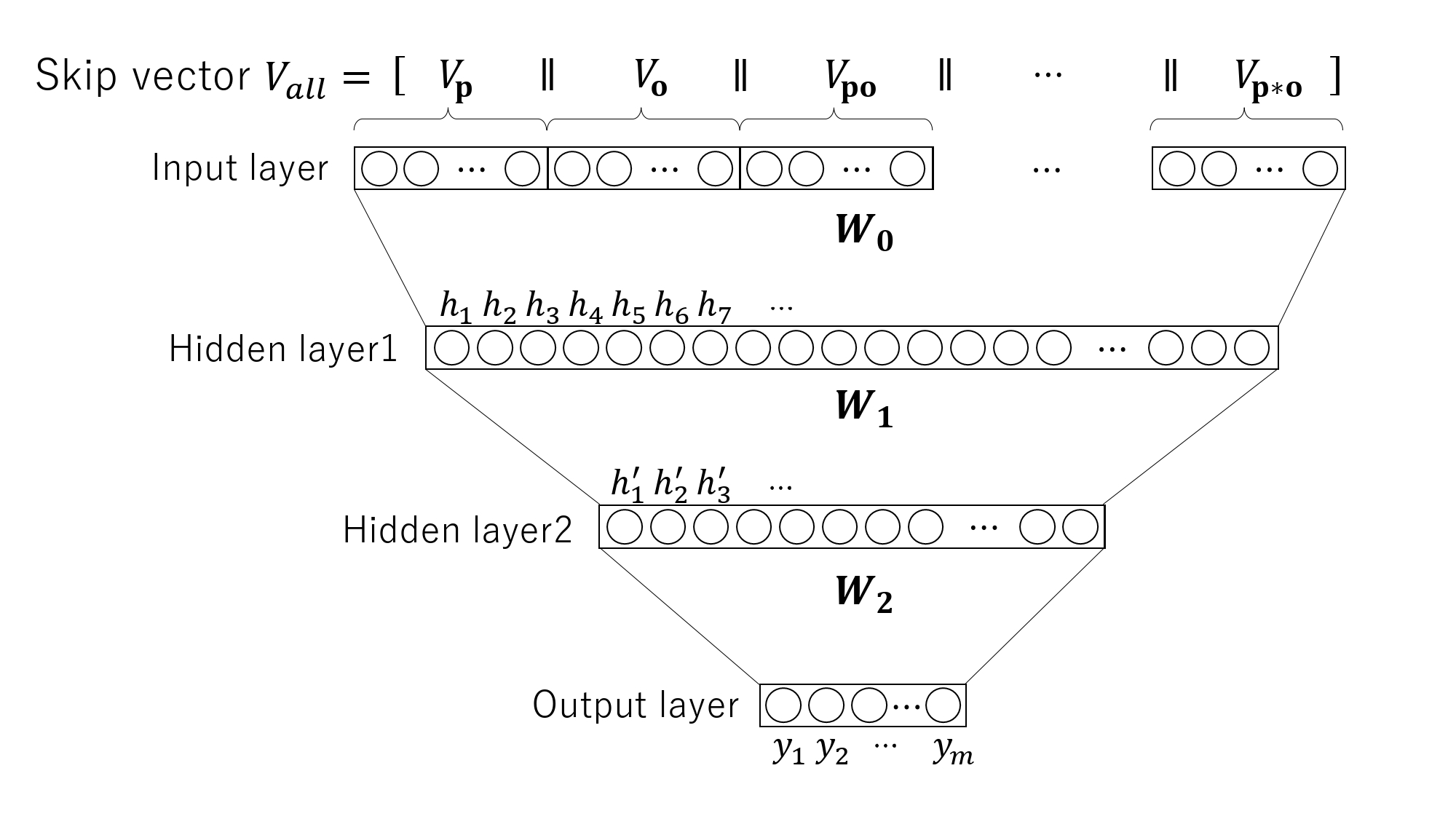}
		\caption{Neural network architecture for Skip vector $V_{all}$}
		\label{fig:dnn}
	\end{figure}
	
	Figure \ref{fig:dnn} depicts a neural network architecture for $m$-class classification tasks, where the vector $V_{all}$ concatenated from the Skip vectors $V_{\mathit{sp}}$ of all Skip patterns $\mathit{sp} \in \mathit{SP}$ is applied to the input layer.
	
	\subsection{Feature Selection by Information Gain Ratio}
	\label{sec:section}
	Depending on the Skip pattern, the dimensionality of Skip vectors may be too large, making the machine learning computational requirement unrealistically high. 
	Therefore, we use the information gain ratio to reduce the dimensionality of Skip vectors by selecting features that are useful for classification tasks.
	
	Let $F_{\mathit{sp}}(\rSet{G}) = \{f_1,\dots,f_k\}$ be the feature set of a Skip pattern $\mathit{sp}\in \mathit{SP}$.
	We define $F_{\mathit{sp}}^{\epsilon_n}(\rSet{G}) = \{f_1^\prime,\dots,f_n^\prime\} (\subseteqq F_{\mathit{sp}}(\rSet{G}))$ as the set of features selected from the top $n(\le k)$ of all features in the order of increasing information gain ratio $\igrfl{\rSet{G}}{f_1},\dots,\igrfl{\rSet{G}}{f_k}$.
	The Skip vector for resource $s\in \rSet{G}$ created from the feature set $F_{\mathit{sp}}^{\epsilon_n}(\rSet{G})$ is denoted as $V_{\mathit{sp}}^{\epsilon_n}(s)$.
	Furthermore, $V_{all}^{\epsilon_n}(s) = \parallel_{\mathit{sp}\in \mathit{SP}} V_{\mathit{sp}}^{\epsilon_n}(s)$ denotes the concatenation of the Skip vectors $V_{\mathit{sp}}^{\epsilon_n}(s)$ created from the top $n$ feature sets $F_{\mathit{sp}}^{\epsilon_n}(\rSet{G})$ for all Skip patterns $\mathit{sp}\in \mathit{SP}$.
	
	We exploit a discount factor $\lambda$ to reduce the effect of features further from the resource on the similarity between resources.
	\begin{dfn}[Discount factor for Skip patterns]
		For each increase in the depth $d$ of a Skip pattern $\mathit{sp}$, each value $x_{i}^{\mathit{sp}}$ of the Skip vector $V_{\mathit{sp}}(s)$ for resource $s\in \rSet{G}$ is multiplied by $\lambda$ to $\lambda^{d-1} \cdot x_{i}^{\mathit{sp}}$.
	\end{dfn}
	
	For example, if the discount factor $\lambda = 0.3$, the value $x_{i}^{\mathit{sp}}$ of the Skip vector $V_{\mathit{sp}}$ for each Skip pattern $\mathit{sp}\in \{\textbf{p}, \textbf{o}, \textbf{po}\}$ at depth 1 is multiplied by $\lambda^0 = 1$, and the value $x_{i}^{\mathit{sp}}$ of the Skip vector $V_{\mathit{sp}}$ for each Skip pattern $\mathit{sp}\in \{*\textbf{p},*\textbf{o},*\textbf{po},\textbf{pp},\textbf{ppo}, \textbf{p}\text{\textasteriskcentered} \textbf{o}\}$ at depth 2 is multiplied by $\lambda^1 = 0.3$.
	
	The following summarizes the procedure for generating the Skip vectors from an RDF graph whose features are selected by the information gain ratio.
	\begin{description}
		\item[(i)] First, the feature set $F_{\mathit{sp}}(\rSet{G})$ for all resources is generated by exploring an RDF graph.
		\item[(ii)] Next, the information gain ratio of each feature $f\in F_{\mathit{sp}}(\rSet{G})$ is calculated from the resource set $\rSet{G}^c$ for each class label $c \in C$.
		\item[(iii)] Finally, the feature set $F_{\mathit{sp}}^{\epsilon_n}(\rSet{G})$ is selectively extracted through the information gain ratio of each feature in $F_{\mathit{sp}}(\rSet{G})$, and then the Skip vector $V_{\mathit{sp}}^{\epsilon_n}(s)$ of $F_{\mathit{sp}}^{\epsilon_n}(s)$ for each resource $s\in \rSet{G}$ is generated.
	\end{description}
	
	\section{Experiments}
	
	We evaluate the Skip vectors generated from RDF data by applying them to conventional machine learning algorithms.
	In the evaluation experiments, we compare our method with existing RDF graph kernel methods, as well as RDF graph embeddings and relational graph convolutional networks.
	
	The process of generating the Skip vectors from an RDF graph was implemented in Java 11. 
	We efficiently searched an RDF graph in the RDF datastore FROST 1.1.1.\footnote{\url{http://www.sw.cei.uec.ac.jp/frost/}}
	We implemented neural networks in Keras; SVMs, RFs, KNN, ADA in scikit-learn~\cite{scikit-learn}; and GCN in TensorFlow 2.1.0.
	
	\subsection{Comparison with RDF Graph Kernels}
	
	\resTableC
	\subsubsection{Datasets}
	We summarize eight binary classification problems for the RDF datasets Wikidata, DBpedia, and YAGO as follows.
	
	\begin{description}
		\item[Gender1]Classification of gender (284464 triples, 2 classes, 200 resources)
		\item[SeaLake]Classification of sea and lake (177340 triples, 2 classes, 100 resources)
		\item[NetIncome]Classification of companies by net income (104523 triples, 2 classes, 100 resources)
		\item[ScientistArtist]Classification of scientists and artists (46822 triples, 2 classes, 200 resources)
		\item[BoxOffice]Classification of movies by box office revenue (64521 triple, 2 classes, 200 resources)
		\item[Gender2]Classification of gender (40724 triples, 2 classes, 200 resources)
		\item[GDP]Classification of countries by gross domestic product (108517 triples, 2 classes, 100 resources)
		\item[PopulationDensity]Classification of locations by population density (51748 triples, 2 classes, 200 resources)
	\end{description}
	
	Gender1, SeaLake, and NetIncome are datasets extracted from Wikidata;\footnote{\url{https://www.wikidata.org/}} ScientistArtist and BoxOffice are datasets extracted from DBpedia Japanese;\footnote{\url{http://http://ja.dbpedia.org/}} Gender2, GDP, and PopulationDensity are datasets extracted from YAGO.\footnote{\url{https://www.mpi-inf.mpg.de/departments/databases-and-information-systems/research/yago-naga/yago/}}
	Table \ref{tab:resultC} shows the number of features $|F_{\mathit{sp}}(\rSet{G})|$ for each Skip pattern $\mathit{sp} \in \mathit{SP}$ (before feature selection).
	For all datasets, we perform 10-fold cross-validation with 90\% of target resources as training data and 10\% as test data.
	
	\subsubsection{Experimental Setting}
	\label{4.1.2}
	We generate the Skip vectors $V_{all}^{\epsilon_n}$, where the number of features for each Skip pattern $\mathit{sp}\in \mathit{SP}$ is limited to $n$, by selecting features using the information gain ratio.
	We apply the conventional machine learning algorithms SVM, KNN, neural networks (NN), RF, and ADA.
	For each of the algorithms, the number of features $n$ is set to $n_{\scalebox{0.5}{SVM}},n_{\scalebox{0.5}{KNN}}=100$, $n_{\scalebox{0.5}{NN}}=150$, $n_{\scalebox{0.5}{RF}}=350$, and $n_{\scalebox{0.5}{ADA}}=450$.
	For SVM, KNN, and NN, the discount factor $\lambda$ varies from 0.1 to 1.0 in increments of 0.1.
	The features occurring once in training data are considered to be noise data and are excluded.
	
	For SVM, we use a regularization parameter of 1.0 and the hinge loss function.
	For KNN, we select the value k=5.
	For NN, we use the ReLU activation function on the input, at most four hidden layers, and the softmax activation function on the output layer. The hidden layer size is reduced from the previous layer size to $\frac{1}{5}$, and the output layer size is two. We train the neural network model for 200 epochs by minimizing the loss function categorical cross entropy with the $Adam$ optimizer and a mini-batch size of 32. We stop training if the loss function does not decrease in 10 epochs.
	For RF, we use the Gini impurity as the splitting criterion, and set the number of decision trees to \{$5$, $10$, $20$, $40$, $60$, $80$, $100$, $150$, $200$, $300$\}.
	For ADA, the maximum depth of decision trees, as weak classifiers, is set from 1 to 10.
	
	\resTableA
	\subsubsection{Results}
	In Table \ref{tab:resultA}, we provide the results of classification accuracies for Skip vectors in the conventional machine learning algorithms.
	We compare our method against the classification accuracies for the RDF graph kernels Skip, Hop~\cite{arai:18}, PRO~\cite{arai:17}, Walk, Path, Full Subtree, and Partial Subtree~\cite{losch12} in SVM, reported in a previous work~\cite{arai:18}.
	The discount factor $\lambda$ is set to $0.0001$, $0.001$, $0.01$, and $0.1$ for the Partial Subtree kernel, and from 0.1 to 1.0 in increments of 0.1 for the other kernels.
	
	The performance of Skip vectors, even with the conventional machine learning algorithms, is superior to the existing RDF graph kernels.
	The best accuracies of our method are equal to or better than the best of all graph kernel methods on 7 out of 8 datasets (i.e., Gender1, SeaLake, NetIncome, ScientistArtist, BoxOffice, Gender2, PopulationDensity).
	In particular, the best accuracies of Gender1, NetIncome, and PopulationDensity are improved by 0.5\%--4\%.
	
	\subsection{Comparison with RDF2Vec and R-GCN}
	\label{4.2}
	
	\subsubsection{Datasets}
	We summarize multiclass classification problems for the four dataset benchmarks~\cite{schlichtkrull2018modeling}.\footnote{\url{https://github.com/tkipf/relational-gcn}}
	
	\begin{description}
		\item[AIFB]Classification of human affiliations in the Institute of Applied Informatics and Formal Description Methods (29043 triples, 4 classes, 176 resources)
		\item[MUTAG]Classification of potentially carcinogenic complex molecules (74227 triples, 2 classes, 340 resources)
		\item[BGS]Classification of rocks by the British Geological Survey (916199 triples, 2 classes, 146 resources)
		\item[AM]Classification of artifacts from the Amsterdam Museum (5988321 triples, 11 classes, 1000 resources)
	\end{description}
	
	By following the experimental setting in R-GCN~\cite{schlichtkrull2018modeling}, the RDF triples that were used to create the class labels for target resources are deleted in each dataset.	
	For all datasets, we split the target resources into 80\% and 20\% for training and testing, respectively.
	
	\resTableK
	
	\subsubsection{Experimental Setting}
	We use the Skip vectors $V_{all}^{\epsilon_n}$ with the number of features for each Skip pattern $n = 1000$. 
	The features occurring once in the training data are considered as noise data and are excluded. 
	For SVM, KNN, NN, and GCN, we choose the discount factor $\lambda$ $\in$ \{0.1, 0.2, 0.3, 0.4 ,0.5, 0.6, 0.7, 0.8, 0.9, 1.0\} based on the validation set performance.\footnote{The maximum value is chosen if many best parameters based on the validation set exist.}
	
	For SVM, we use a regularization parameter of 1.0 and the linear kernel. For KNN, we select the value k=5. For NN, we use the ReLU activation function on the input, at most four hidden layers, and the softmax activation function on the output layer. The hidden layer size is reduced from the previous layer size to $\frac{1}{5}$, and the output layer size is two. We train the neural network model for 200 epochs by minimizing the loss function categorical cross entropy with the $Adam$ optimizer and a mini-batch size of 32. We stop training if the loss function does not decrease in 10 epochs. For RF, we use the Gini impurity as the splitting criterion, and choose the number of decision trees to \{$5$, $10$, $20$, $40$, $60$, $80$, $100$, $150$, $200$, $300$\} based on validation set performance.\footnotemark[6] For ADA, the maximum depth of decision trees, as weak classifiers, is chosen from 1 to 10 based on validation set performance.\footnotemark[6] 
	
	For GCN, we use the ReLU activation function on the input and intermediate layers, and the softmax activation function on the output layer, where the intermediate layer size is 16, the optimizer is $Adam$, the L2 regularization weight is $5e$--$4$, the dropout rate is 0.5, and the learning rate is 0.01. For 100 epochs, we train the GCN model, where the feature matrix is constructed by the Skip vectors, and the adjacency matrix is derived by converting an RDF graph to an undirected graph.
	
	\subsubsection{Results}	
	We compare our Skip vectors with RDF2Vec and R-GCN.
	The classification accuracies for applying the Skip vectors to SVM, KNN, NN, RF, ADA, and GCN are displayed in Table \ref{tab:resultK}.
	We have not shown the results on the datasets BGS and AM for SkipVec-GCN due to out of memory. 
	The experimental results for Feat~\cite{paulheim2012unsupervised}, WL (Weisfeiler-Lehman kernels)~\cite{shervashidze2011weisfeiler}, RDF2Vec~\cite{ristoski16}, and R-GCN~\cite{schlichtkrull2018modeling} are shown in the bottom part of Table \ref{tab:resultK}, reported in a previous work~\cite{schlichtkrull2018modeling}.
	
	The average accuracies for SkipVec-SVM and SkipVec-RF on all datasets outperform the previous methods. 
	In particular, the best accuracies for SkipVec-GCN on the AIFB dataset, SkipVec-SVM, SkipVec-RF, and SkipVec-ADA on the datasets BGS, and SkipVec-RF on the dataset AM outperform the previous methods on the AIFB, BGS, and AM datasets.
	These results indicate that the Skip vectors successfully represent the features of each resource in RDF graphs and achieve high accuracy in combination with conventional machine learning algorithms and the basic GCN.
	
	\section{Related Works}
	
	\subsection{Graph Kernels}
	There are several works on graph kernels~\cite{fanizzi2006declarative, huang2014scalable, collins:01, shervashidze2011weisfeiler, vishwanathan2010graph, kashima2009, kang2012fast} that apply tree or graph structure data without the input of vectors to machine learning tasks.
	The graph kernels are defined by functions that calculate the distance between data by counting the common substructures for two graphs or nodes, enabling machine learning on graph data.
	For example, an SVM can be trained to classify two graphs or nodes by calculating the similarity between data in the graph kernel instead of the inner product of vectors.
	
	\subsection{RDF Graph Kernels}
	L{\"{o}}sch et al.~\cite{losch12} have proposed the graph kernels Walk, Path, Full Subtree, and Partial Subtree for specializing in RDF graphs.
	Based on intersection graphs, the Walk kernel and Path kernel count the number of common walks and paths for two graphs, respectively.
	In addition, based on intersection trees, the Full Subtree kernel and Partial Subtree kernel count the number of common full and partial subtrees for two graphs, respectively.
	
	By extending these RDF graph kernels, Arai et al.~\cite{arai:17, arai:18} have proposed the PRO kernel, which eliminates all intermediate nodes (objects) from Walks, and the Skip kernel, which arbitrarily eliminates edges (predicates) and nodes (objects) from PROs.
	The Skip kernel improved classification accuracy compared with the other RDF graph kernels.
	
	However, the use of the RDF graph kernels is limited to be incorporated in machine learning algorithms such as SVM.
	On the other hand, the Skip vectors can be used as the input of feature vectors for applying most conventional machine learning algorithms.
	
	\subsection{RDF Graph Embeddings}
	Related to learning with RDF data, an approach to RDF graph embeddings has been studied.
	RDF2Vec~\cite{ristoski16} embeds target resources in an RDF graph into a low-dimensional vector space, inspired by the word embeddings of word2vec~\cite{mikolov2013efficient}.
	Given a sequence of words, Word2vec trains the continuous bag-of-words (CBOW) model to predict the target word from surrounding words (context words), and the skip-gram model to predict surrounding words (context words) from the target word.
	Instead of word sequences in word2vec, RDF2Vec learns the vector representations of resources in an RDF graph by applying graph structures such as Walks and Subtrees to the CBOW and skip-gram models.
	
	The Skip vectors do not learn RDF graph embeddings as proposed in RDF2Vec, but generate feature vectors by extracting and selecting substructures from an RDF graph.
	
	\subsection{Relational Graph Convolutional Networks}
	GCN~\cite{kipf2016semi} performs semi-supervised learning on a graph dataset with the feature vectors for all nodes.
	For example, there is a dataset that consists of a graph network representing the relationships among papers and the feature vectors of papers created by the bag-of-words from the documents of each paper.
	Furthermore, R-GCN~\cite{schlichtkrull2018modeling} has been extended as a variant of GCN to relational graph data, such as RDF graphs.
	
	Unlike learning the R-GCN model, the Skip vectors aim at generating feature vectors from an RDF graph that can generally be used in combination with other learning models and methods, such as SkipVec-RF, SkipVec-ADA, and SkipVec-GCN, as described in Section \ref{4.2}.
	
	\section{Conclusion}
	\label{conclusion}
	In this paper, we proposed a method for generating the Skip vectors on various feature patterns (Skip patterns) by extracting predicates and objects from target resources in an RDF graph.
	In particular, we have theoretically proved the complexity of Skip patterns to limit the exponential combinations of features in RDF graphs.
	In addition, the dimensionality of Skip vectors is reduced by selecting features using the information gain ratio.
	In the experiments for classification tasks using conventional machine learning algorithms, we achieved better performance in comparison with the existing RDF graph kernels, RDF graph embeddings, and relational graph convolution networks.
	Therefore, for classification tasks in RDF data, Skip vectors successfully represent the features of target resources and can be combined with SVM, KNN, NN, RF, ADA, and the basic GCN.
	
	In future work, we plan to apply Skip vectors to link prediction in knowledge graphs and logical reasoning with ontology embeddings.
	
	
	
	\bibliographystyle{plain}
	\bibliography{bibfile_e}

\begin{thebibliography}{10}

\bibitem{arai:17}
Daichi Arai and Ken Kaneiwa.
\newblock A kernel function for redundant features from {RDF} graphs and its
  fast calculation.
\newblock {\em Transactions of the Japanese Society for Artificial
  Intelligence}, 32(1):B--G34\_1--12, 2017(in Japanese).

\bibitem{arai:18}
Daichi Arai and Ken Kaneiwa.
\newblock A generic kernel for various {RDF} graphs.
\newblock {\em Transactions of the Japanese Society for Artificial
  Intelligence}, 33(5):B--I12\_1--14, 2018(in Japanese).

\bibitem{bicer:11}
Veli Bicer, Thanh Tran, and Anna Gossen.
\newblock Relational kernel machines for learning from graph-structured {RDF}
  data.
\newblock In {\em Proceedings of the 8th Extended Semantic Web Conference,
  ({ESWC} 2011)}, pages 47--62, 2011.

\bibitem{collins:01}
Michael Collins and Nigel Duffy.
\newblock Convolution kernels for natural language.
\newblock In {\em Proceedings of the Neural Information Processing Systems
  ({NIPS} 14)}, pages 625--632, 2001.

\bibitem{exner:12}
Peter Exner and Pierre Nugues.
\newblock Entity extraction: From unstructured text to dbpedia {RDF} triples.
\newblock In {\em Proceedings of the Web of Linked Entities Workshop ({WoLE}
  2012)}, pages 58--69, 2012.

\bibitem{fanizzi2006declarative}
Nicola Fanizzi and Claudia d'Amato.
\newblock A declarative kernel for \emph{ALC} concept descriptions.
\newblock In {\em Proceedings of the 16th International Symposium on
  Methodologies for Intelligent Systems ({ISMIS} 2006)}, pages 322--331, 2006.

\bibitem{fanizzi:12}
Nicola Fanizzi, Claudia d'Amato, and Friana Esposito.
\newblock Induction of robust classifiers for web ontologies through kernel
  machines.
\newblock {\em J. Web Semant.}, 11:1--13, 2012.

\bibitem{kashima2009}
Shohei Hido and Hisashi Kashima.
\newblock A linear-time graph kernel.
\newblock In {\em Proceedings of the 9th {IEEE} International Conference on
  Data Mining ({ICDM} 2009)}, pages 179--188, 2009.

\bibitem{huang2014scalable}
Yi~Huang, Volker Tresp, Maximilian Nickel, Achim Rettinger, and Hans{-}Peter
  Kriegel.
\newblock A scalable approach for statistical learning in semantic graphs.
\newblock {\em Semantic Web}, 5(1):5--22, 2014.

\bibitem{kang2012fast}
U~Kang, Hanghang Tong, and Jimeng Sun.
\newblock Fast random walk graph kernel.
\newblock In {\em Proceedings of the 12th {SIAM} International Conference on
  Data Mining ({SDM})}, pages 828--838, 2012.

\bibitem{kipf2016semi}
Thomas~N. Kipf and Max Welling.
\newblock Semi-supervised classification with graph convolutional networks.
\newblock In {\em Proceedings of the 5th International Conference on Learning
  Representations ({ICLR} 2017)}, 2017.

\bibitem{losch12}
Uta L{\"{o}}sch, Stephan Bloehdorn, and Achim Rettinger.
\newblock Graph kernels for {RDF} data.
\newblock In {\em Proceedings of the 9th Extended Semantic Web Conference
  ({ESWC} 2012)}, pages 134--148, 2012.

\bibitem{mikolov2013efficient}
Tom{\'{a}}s Mikolov, Kai Chen, Greg Corrado, and Jeffrey Dean.
\newblock Efficient estimation of word representations in vector space.
\newblock In {\em Proceedings of the 1st International Conference on Learning
  Representations ({ICLR} 2013)}, 2013.

\bibitem{paulheim2012unsupervised}
Heiko Paulheim and Johannes F{\"{u}}rnkranz.
\newblock Unsupervised generation of data mining features from linked open
  data.
\newblock In {\em Proceedings of the 2nd International Conference on Web
  Intelligence, Mining and Semantics, ({WIMS} '12)}, pages 31:1--31:12, 2012.

\bibitem{scikit-learn}
F.~Pedregosa, G.~Varoquaux, A.~Gramfort, V.~Michel, B.~Thirion, O.~Grisel,
  M.~Blondel, P.~Prettenhofer, R.~Weiss, V.~Dubourg, J.~Vanderplas, A.~Passos,
  D.~Cournapeau, M.~Brucher, M.~Perrot, and E.~Duchesnay.
\newblock Scikit-learn: Machine learning in {P}ython.
\newblock {\em Journal of Machine Learning Research}, 12:2825--2830, 2011.

\bibitem{ristoski16}
Petar Ristoski, Jessica Rosati, Tommaso~Di Noia, Renato~De Leone, and Heiko
  Paulheim.
\newblock {RDF}2{V}ec: {RDF} graph embeddings and their applications.
\newblock {\em Journal of Semantic Web}, 10(4):721--752, 2019.

\bibitem{schlichtkrull2018modeling}
Michael~Sejr Schlichtkrull, Thomas~N. Kipf, Peter Bloem, Rianne van~den Berg,
  Ivan Titov, and Max Welling.
\newblock Modeling relational data with graph convolutional networks.
\newblock In {\em Proceedings of the European Semantic Web Conference ({ESWC}
  2018)}, pages 593--607, 2018.

\bibitem{10.1007/978-3-319-11964-9_16}
Max Schmachtenberg, Christian Bizer, and Heiko Paulheim.
\newblock Adoption of the linked data best practices in different topical
  domains.
\newblock In {\em Proceedings of the 13th International Semantic Web Conference
  ({ISWC} 2014)}, pages 245--260, 2014.

\bibitem{shervashidze2011weisfeiler}
Nino Shervashidze, Pascal Schweitzer, Erik~Jan van Leeuwen, Kurt Mehlhorn, and
  Karsten~M. Borgwardt.
\newblock Weisfeiler-lehman graph kernels.
\newblock {\em Journal of Machine Learning Research}, 12(77):2539--2561, 2011.

\bibitem{vishwanathan2010graph}
S.V.N. Vishwanathan, Nicol~N. Schraudolph, Risi Kondor, and Karsten~M.
  Borgwardt.
\newblock Graph kernels.
\newblock {\em Journal of Machine Learning Research}, 11(40):1201--1242, 2010.

\end{thebibliography}

	\clearpage

\end{document}